\def\year{2020}\relax
\documentclass[letterpaper]{article} 
\usepackage{aaai20}  
\usepackage{times}  
\usepackage{helvet} 
\usepackage{courier}  
\usepackage[hyphens]{url}  
\usepackage{graphicx} 
\usepackage{graphicx}  
\usepackage{xcolor}
\usepackage{mathtools, amsfonts, mathrsfs, algorithm}
\usepackage[noend]{algpseudocode}

\usepackage{amsmath, amsthm, amssymb, stmaryrd}
\usepackage{makecell, multirow}
\usepackage{proof}
\usepackage{adjustbox}
\usepackage{verbatim}
\usepackage{listings}
\usepackage[italicdiff]{physics} 
\usepackage{dsfont} 
\usepackage{subcaption}
\usepackage{booktabs}
\graphicspath{ {./images/} }

\usepackage[scaled]{beramono}
\usepackage{makecell}

\allowdisplaybreaks

\newtheorem{theorem}{Theorem}

\newtheorem{proposition}[theorem]{Proposition}

\renewcommand{\hat}{\widehat}
\renewcommand{\bar}[1]{\mkern 1.5mu\overline{\mkern-1.5mu#1\mkern-1.5mu}\mkern 1.5mu}

\newcommand{\R}{\mathbb{R}}

\newcommand{\defeq}{\triangleq} 

\newcommand{\ph}   {\leavevmode \kern.06em\vbox{\hrule width.8em}} 

\newcommand{\cH}{\mathcal{H}}
\newcommand{\cO}{\mathcal{O}}
\newcommand{\cX}{\mathcal{X}}

\newcommand{\cZ}{\mathcal{Z}}

\newcommand{\ELBO}{\mathrm{ELBO}}
\newcommand{\MaRep}{\mathrm{MRep}}
\newcommand{\algo}{\mbox{\sc Alg}}

\newcommand{\HMCnaive}{\mbox{$\textsf{HMC}_{\textsf{naive}}$}}
\newcommand{\HMCours}{\mbox{$\textsf{HMC}_{\textsf{ours}}$}}
\newcommand{\IPMCMC}{\mbox{\sf IPMCMC}}
\newcommand{\LMH}{\mbox{\sf LMH}}

\newcommand{\cmt}[2]{\textcolor{blue}{{\bf [}{\bf #1:} {\it #2}{\bf ]}}}

\newcommand*{\ks}[1]{\cmt{KS}{#1}}

\newcommand*{\wy}[1]{\cmt{WY}{#1}}
\newcommand{\commentout}[1]{}

\newcommand{\citet}[1]{\citeauthor{#1} \shortcite{#1}}
\newcommand{\citep}{\cite}

\title{Differentiable Algorithm for Marginalising Changepoints}

\author{Hyoungjin Lim, \ Gwonsoo Che, \ Wonyeol Lee, \ Hongseok Yang \\
	School of Computing\\
	KAIST, South Korea\\
	\{lmkmkr, gche, wonyeol, hongseok.yang\}@kaist.ac.kr}

\begin{document}
\maketitle


\begin{abstract}
We present an algorithm for marginalising changepoints in time-series models that assume a fixed number of unknown changepoints. 
Our algorithm is differentiable with respect to its inputs, which are the values of latent random variables other than 
changepoints. Also, it runs in time $\cO(mn)$ where $n$ is the number of time steps and $m$ the number of changepoints, an 
improvement over a naive marginalisation method with $\cO(n^m)$ time complexity. 
We derive the algorithm by identifying quantities related to this marginalisation problem, showing that these quantities satisfy recursive relationships, and transforming the relationships to an algorithm via dynamic programming. Since our algorithm is differentiable, it can be applied to convert a model non-differentiable due to changepoints to a differentiable one, so that the resulting models can be analysed using gradient-based inference or learning techniques. 
We empirically show the effectiveness of our algorithm in this application 
by tackling the posterior inference problem on synthetic and real-world data.
\end{abstract}


\section{Introduction}


Time-series data from, for instance, econometrics, medical science, and political science \citep{erdman2008fast,lio2000wavelet,spokoiny2009multiscale,kaylea2017computationally,reeves2007review,fong2012distributed} often show abrupt regime shifts, so that analysing those data commonly requires reasoning about the moments of these shifts, called changepoints. Two popular reasoning tasks are inferring the number of changepoints and detecting the specific values or distributions of the changepoints. Information found from these tasks enables the use of different statistical models for different segments of the data, identified by changepoints, which leads to accurate analysis of the data. However, due to the discrete nature of changepoints, developing efficient algorithms for the tasks is tricky, and often requires an insight into the structure of a class of models used.

In the paper, we study the problem of marginalising changepoints, which has been under-explored compared with the two tasks mentioned above. We present a differentiable algorithm for marginalising changepoints for a class of time-series models that assume a fixed number of changepoints. Our algorithm runs in $\cO(mn)$ time where $m$ is the number of changepoints and $n$ the number of time steps. We do not know of any $\cO(mn)$-time algorithm that directly solves this changepoint-marginalisation problem. 
%
%
%
The class of models handled by our algorithm is broad, including non-Markovian time-series models.

Our marginalisation algorithm is differentiable with respect to its inputs, which enables the use of gradient-based algorithms for posterior inference and parameter learning on changepoint models. Since changepoints are discrete, gradient-based algorithms cannot be applied to these models, unless changepoints are marginalised out. In fact, marginalising discrete variables, such as changepoints, is a trick commonly adopted by the users of the Hamiltonian Monte Carlo algorithm or its variant \citep{StanUsersGuide}. Our algorithm makes the trick a viable option for changepoint models. Its $\cO(mn)$ time complexity ensures low marginalisation overhead. Its differentiability implies that the gradients of marginalised terms can be computed by off-the-shelf automated differentiation tools \citep{paszke2017automatic,tensorflow2016-osdi}. In the paper, we demonstrate these benefits of our algorithm for posterior inference.

The key insight of our algorithm is that the likelihood of latent variables with changepoints marginalised out can be expressed in terms of quantities that satisfy recursive relationships. The algorithm employs dynamic programming to compute these quantities efficiently. Its $\cO(mn)$ time complexity comes from this dynamic programming scheme, and its differentiability comes from the fact that dynamic programming uses only differentiable operations. In our experiments with an inference problem, the algorithm outperforms existing alternatives.

The rest of the paper is organised as follows. 
In \S\ref{sec:result}, we describe our algorithm and its theoretical properties, and in \S\ref{sec:learning}, we explain how this algorithm can be used to learn model parameters from given data. In \S\ref{sec:empirical}, we describe our experiments where we apply the algorithm to a posterior-inference problem. In \S\ref{sec:related},
we put our results in the context of existing work on changepoint models and differentiable algorithms, and conclude the paper.


\section{Marginalisation Algorithm}
\label{sec:result}

Let $n,m$ be positive integers with $n \geq m$, and $\mathbb{R}_+$ be the set of positive real numbers. We consider a probabilistic model for $n$-step time-series data with $m{+}1$ changepoints, which has the following form. Let
$\cX \subseteq \mathbb{R}^k$ and $\cZ \subseteq \mathbb{R}^l$.
\begin{align*}
        x_{1:n} \in \cX^n 
        & \text{ --- data points over $n$ time steps.} 
        \\ 
        w_{1:n} \in \mathbb{R}^n_{+} 
        & \text{ --- }
        \begin{array}[t]{@{}l@{}}
                \text{$w_t$ expresses a relative chance of the}\\
                \text{step $t$ becoming a changepoint. $w_n = 1$.}
        \end{array}
        \\ 
        z_{1:m} \in \cZ^m 
        & \text{ --- }
        \begin{array}[t]{@{}l@{}}
                \text{latent parameters deciding the distribution}\\
                \text{of the data points $x_{1:n}$.}
        \end{array}
        \\ 
        \tau_{0:m} \in \mathbb{N}^{m+1}
        & \text{ --- changepoints. $0 \,{=}\, \tau_0 \,{<}\, \tau_1 \,{<} \cdots {<}\, \tau_m \,{=}\, n$.} 
\end{align*}
\begin{align*}
        & P(\tau_{0:m}\,|\,w_{1:n}) 
        \defeq
        \frac{1}{W} \prod_{i=1}^{m-1} w_{\tau_i}
        \quad
        \text{ where $W = \sum_{\tau_{0:m}} \prod_{i=1}^{m-1} w_{\tau_i}$},
        \\
        &
        P(x_{1:n},z_{1:m},\tau_{0:m} \,|\,w_{1:n}) 
        \defeq
        {}
        \\
        & \quad\
        P(z_{1:m})
        P(\tau_{0:m} \,|\,w_{1:n})
        \prod_{i=1}^m \prod_{(j=\tau_{i-1} + 1)}^{\tau_i} P(x_j\,|\,x_{1:j-1},z_i).
\end{align*}
For simplicity, we assume for now that $w_{1:n}$ is fixed and its normalising constant $W$ is known. In \S\ref{sec:norm-const}, we will show how the assumption can be removed safely.

Our goal is to find an efficient algorithm for computing 
the likelihood of the data $x_{1:n}$ for the latent $z_{1:m}$, which involves marginalising the changepoints $\tau_{0:m}$
as shown below:
\begin{equation}\label{eqn:goal}
        P(x_{1:n} \,|\, z_{1:m}, w_{1:n})
        \,{=}\, \sum_{\tau_{0:m}} P(x_{1:n},\tau_{0:m}\,|\,z_{1:m}, w_{1:n}).
\end{equation}
Note that summing the terms in \eqref{eqn:goal} naively is not a viable option because the number of the terms grows exponentially in the number of changepoints (i.e., $\cO(n^{m-1})$).

Our algorithm computes the sum in \eqref{eqn:goal} in $\cO(mn)$ time. Two key ideas behind the algorithm
are to rewrite the sum in terms of recursively expressible quantities, and to compute these quantities efficiently using dynamic programming. 

For integers $k,t$ with $1\leq k < m$ and $m-1 \leq t < n$,
let
\begin{align*}
        T_{k,t} & \defeq \{ \tau_{0:m} \mid
        \begin{array}[t]{@{}l@{}}
                \tau_{0:m} \text{ changepoints},\ \tau_{m-1} = t,\ \text{and}\\
                1 + \tau_i = \tau_{i+1} \text{ for all $i$ with } k\,{\leq}\, i \,{<}\, m{-}1\},
        \end{array}
        \\
        L_{0,t} & \defeq 0, 
        \qquad
        L_{k,t} \defeq \sum_{\tau_{0:m} \in T_{k,t}} P(x_{1:n},\tau_{0:m} \,|\, z_{1:m},w_{1:n}), 
        \\
        R_{m,t} & \defeq 1,
        \qquad
        b(k,t) \defeq t-((m-1)-k),
        \\
        R_{k,t} & \defeq 
        \prod_{j=k}^{m-1} \frac{P(x_{b(j,t)+1} \,|\,x_{1:b(j,t)}, z_j) \times w_{b(j,t) + 1}}{P(x_{b(j,t)+1} \,|\,x_{1:b(j,t)}, z_{j+1}) \times w_{b(j,t)}}\, .
\end{align*}
The first $T_{k,t}$ consists of changepoints $\tau_{0:m}$ such that $\tau_{m-1}$ ends at $t$ and $(\tau_k,\tau_{k+1},\ldots,\tau_{m-1})$ are consecutive. The next $L_{k,t}$ selects the summands in \eqref{eqn:goal} whose changepoints $\tau_{0:m}$ are in $T_{k,t}$. It then sums the selected terms. The $b(k,t)$ computes the value of the $k$-th changepoint for $1 \leq k < m$ when the changepoints $\tau_{k}$, $\tau_{k+1},\ldots,\tau_{m-1}$ are consecutive and the value of $\tau_{m-1}$ is $t$. The last $R_{k,t}$ is the ratio of the probabilities of the segment $x_{b(k,t){+}1:b(m{-}1,t){+}1}$ ($= x_{b(k,t){+}1:t{+}1}$) and the changepoints $\tau_{k:m-1}$ under two different assumptions. The numerator assumes that $\tau_{j}=b(j,t)+1$
for all $k \leq j < m$, whereas the denominator assumes that $\tau_{j} = b(j,t)$ for all those $j$. A good heuristic is to view $R_{k,t}$ as the change in the probability of the segment $x_{b(k,t)+1:t+1}$ and the changepoints $\tau_{k:m-1}$ when those changepoints are shifted one step to the right.

%
\begin{figure}[t]
        \center{
	\includegraphics[width=.95\columnwidth]{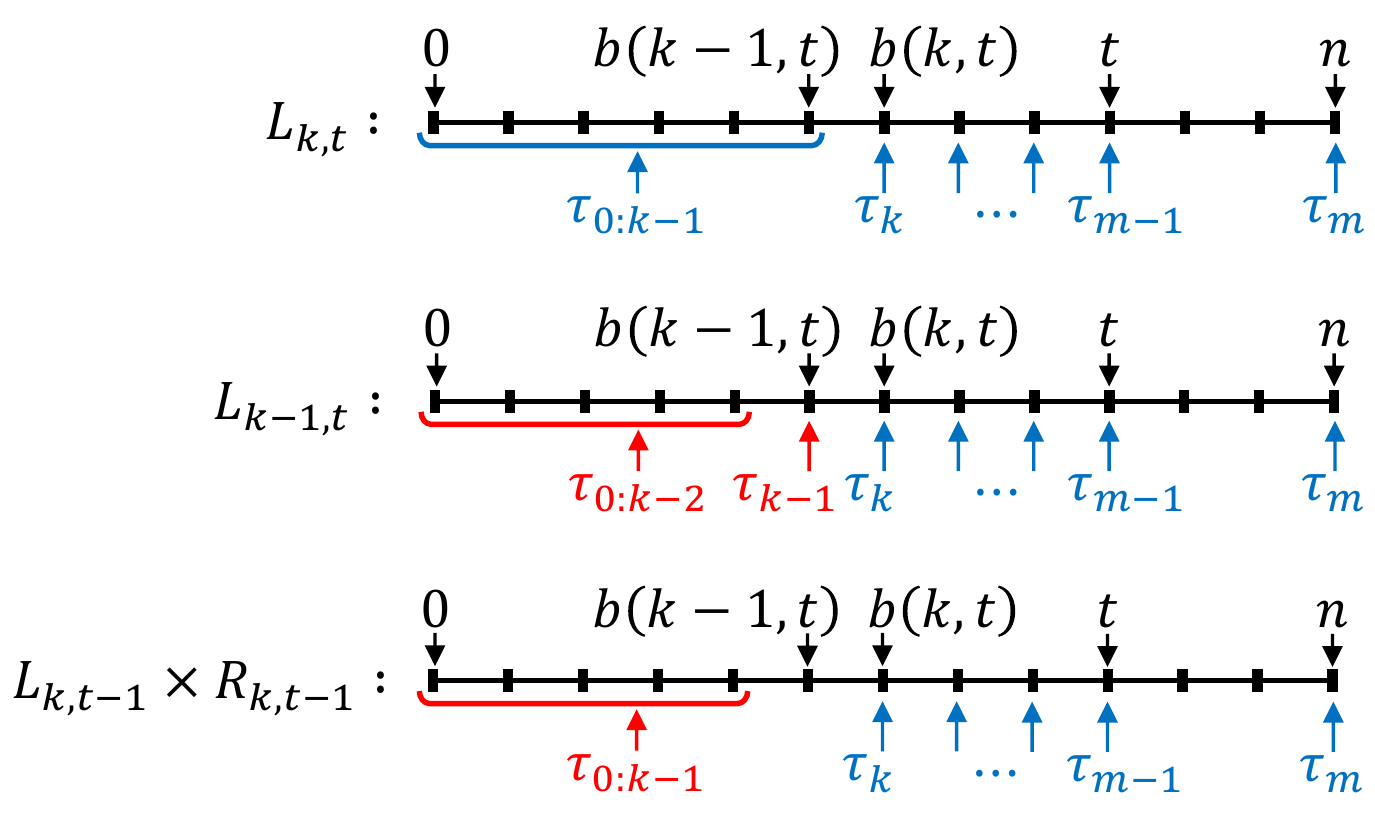}
        }
        \caption{Visualisation of a case split used in Theorem~\ref{thm:algo:recL}.}
        \label{fig:visualisation:split}
\end{figure}

The next three results formally state what we have promised: the $L_{k,t}$ can be used to express the sum in \eqref{eqn:goal}, and the $L_{k,t}$ and the $R_{k,t}$ satisfy recursive relationships.
\begin{proposition}\label{prop:algo:finalL}
        $\sum_{\tau_{0:m}} P(x_{1:n},\tau_{0:m}\,|\, z_{1:m}, w_{1:n})$ in \eqref{eqn:goal} is equal to
        $\sum_{t=m-1}^{n-1} L_{m-1,t}$.
\end{proposition}
\begin{proof}
Because $\{ T_{m-1,t} \,\mid\,m-1 \leq t < n\}$ is a partition of the set of all changepoints $\tau_{0:m}$.
\end{proof}
\noindent{Thus, we can marginalise changepoints by computing $\sum_{t=m-1}^{n-1} L_{m-1,t}$. This begs the question of how to compute the $L_{m-1,t}$'s. The next two results give an answer.}
\begin{theorem}\label{thm:algo:recL}
For all $k,t$ with $1 \leq k < m$ and $m \leq t < n$, 
\begin{align*}
        L_{k,m-1} & = P(x_{1:n} \,|\, z_{1:m},\, \tau_{0:m}{=}(0,1,\ldots,m{-}1,n)) 
        \\
                  & \qquad\qquad {} \times P(\tau_{0:m}{=}(0,1,\ldots,m{-}1,n) \,|\, w_{1:n})\, ,
        \\
        L_{k,t} & = L_{k-1,t} + L_{k,t-1} \times R_{k,t-1}\, .
\end{align*}
\end{theorem}
\noindent
Figure~\ref{fig:visualisation:split} visualises a case split used in the second equation.
$L_{k,t}$ is the quantity about changepoints with $\tau_{k:m-1} = (b(k,t),\ldots,t)$.
The figure shows that such changepoints can be partitioned
into those with $\tau_{k-1} = b(k-1,t)$ and the rest.
The first summand $L_{k-1,t}$ computes the contribution
of the changepoints in the first partition,
and the other summand of the equation that of the changepoints in the second partition.
\begin{proof}
By definition, $T_{k,m-1}$ is the singleton set $\{\tau_{0:m} \,\mid\,\tau_{0:m} = (0,1,\ldots,m-1,n)\}$.
The first equality in the theorem follows from this and the definition of $L_{k,m-1}$. 
For the second equality, consider $k,t$ that satisfy the condition in the theorem. We will prove the following two equations:
\begin{align}
        \label{eqn:Lthm-gen:1}
        & {L_{k,t-1} \times R_{k,t-1}} =
        \!\!\!\!\!\!\!\!\!\!
        \sum_{\substack{\tau_{0:m} \in T_{k,t} \\ \tau_{k{-}1}\neq b(k{-}1,t)}}
        \!\!\!\!\!\!\!\!\!\!
        P(x_{1:n}, \tau_{0:m} \,|\, z_{1:m}, w_{1:n})\, ,
        \\
        \label{eqn:Lthm-gen:2}
        & L_{k-1,t} =
        \sum_{\substack{\tau_{0:m} \in T_{k,t} \\ \tau_{k{-}1} = b(k{-}1,t)}}
        P(x_{1:n}, \tau_{0:m} \,|\, z_{1:m}, w_{1:n})\, .
\end{align}
\noindent{The desired conclusion follows from these two equations:}
\begin{align*}
        & L_{k-1,t} + L_{k,t-1} \times R_{k,t-1}
        \\
        & \qquad{} =
        \sum_{\tau_{0:m} \in T_{k,t}} P(x_{1:n},\tau_{0:m} \,|\, z_{1:m},w_{1:n})
        = L_{k,t}\, .
\end{align*}

Equation \eqref{eqn:Lthm-gen:2} holds since
$T_{k-1,t} = \{\tau_{0:m} \in T_{k,t} \,\mid\,\tau_{k{-}1} = b(k{-}1,t)\}$.
Equation \eqref{eqn:Lthm-gen:1} is proved as follows.
Let
\begin{align*}
        \hat{T}_{k,t} & \defeq \{\tau_{0:m} \in T_{k,t} | \tau_{k-1}\neq b(k-1,t)\},\\
        \tau'_{0:m} & \defeq (\tau_{0:k-1}, \tau_{k}-1, ... , \tau_{m-1}-1, n) \text{ \, for } \tau_{0:m} \in \hat{T}_{k,t}.
\end{align*}
Then, $\{\tau'_{0:m}|\tau_{0:m} \in \hat{T}_{k,t}\}$ = $T_{k,t-1}$.
For every $\tau_{0:m} \in \hat{T}_{k,t}$,
\begin{align*}
    & \frac{P(x_{1:n},\tau_{0:m}|z_{1:m},w_{1:n})}{P(x_{1:n},\tau'_{0:m}|z_{1:m},w_{1:n})}\\
  & = \frac{P(\tau_{0:m}|w_{1:n})}{P(\tau'_{0:m}|w_{1:n})} \times
  \frac{P(x_{1:n}|\tau_{0:m},z_{1:m})}{P(x_{1:n}|\tau'_{0:m},z_{1:m})}\\
    & = \prod_{i=k}^{m-1} \frac{w_{\tau_i}}{w_{\tau_i'}}  \times 
     \prod_{i=k}^{m-1} \frac{P(x_{\tau_i}|x_{1:\tau_i-1},z_i)}{P(x_{\tau'_i+1}|x_{1:\tau'_i},z_{i+1})}\\
    & = \prod_{i=k}^{m-1} \frac{w_{b(i,t)}}{w_{b(i,t)-1}} \times 
    \left. \prod_{i=k}^{m-1} \frac{P(x_{b(i,t)}|x_{1:b(i,t)-1},z_i)}{P(x_{b(i,t)}|x_{1:b(i,t)-1},z_{i+1})}\right.\\
    & = R_{k,t-1}.
\end{align*}
Therefore,
\begin{align*}
  &\sum_{\tau_{0:m} \in \hat{T}_{k,t}} P(x_{1:n},\tau_{0:m} \,|\, z_{1:m}, w_{1:n})\\
  & = \sum_{\tau_{0:m} \in \hat{T}_{k,t}} P(x_{1:n},\tau'_{0:m} \,|\, z_{1:m},w_{1:n}) \times R_{k,t-1}\\
  & = \sum_{\tau_{0:m} \in T_{k,t-1}} P(x_{1:n},\tau_{0:m} \,|\, z_{1:m},w_{1:n}) \times R_{k,t-1} \\
  & = L_{k,t-1} \times R_{k,t-1}.
\end{align*}
\end{proof}

\begin{proposition}\label{prop:algo:recR}
For all $k,t$ with $1 \leq k < m$ and $m-1\leq t < n$,
\[
        R_{k,t} = R_{k+1,t} \times \frac{P(x_{b(k,t)+1} \,|\,x_{1:b(k,t)}, z_k) \times w_{b(k,t) + 1}}{P(x_{b(k,t)+1} \,|\,x_{1:b(k,t)}, z_{k+1}) \times w_{b(k,t)}}\, .
\]
\end{proposition}
\begin{proof}
Immediate from the definition of $R_{k,t}$.         
\end{proof}

\begin{algorithm}[t]
        \caption{Algorithm for marginalising changepoints.\label{algo:marginal}}
        \begin{algorithmic}[1]
                \Require (i) integer $m$; (ii) weights $w_{1:n}$ with $w_n\,{=}\,1$; (iii) normalising constant $W$ for $P(\tau_{0:m}|w_{1:n})$; (iv) latent variables $z_{1:m}$; (v) time-series data $x_{1:n}$ with $1 \,{\leq}\, m \,{\leq}\, n$
                \Ensure likelihood $P(x_{1:n} \,|\,z_{1:m},w_{1:n})$ where changepoints $\tau_{0:m}$ are marginalised
                \For{$t \gets m-1$ to $n-1$}
                  \State $L_{0,t} \gets 0$; $R_{m,t} \gets 1$
                \EndFor
                \State $L_{1,m-1} \gets 
                   \begin{array}[t]{@{}r@{}}
                          P(x_{1:n} | z_{1:m},\tau_{0:m}{=}(0,1,\ldots,m{-}1,n)) \\
                          {} \times P(\tau_{0:m}{=}(0,1,\ldots,m-1,n) | w_{1:n})
                   \end{array}$
                \For{$k \gets 2$ to $m-1$}
                  \State $L_{k,m-1} \gets L_{1,m-1}$
                \EndFor
                \State $P_{i,j} \gets P(x_j | x_{1:j-1}, z_i)$
                for $1 \leq i \leq m$ and $2 \leq j \leq n$
                \For{$t \gets m-1$ to $n-1$}
                  \For{$k \gets m-1$ downto $1$}
                    \State $R_{k,t} \gets 
                    \begin{array}[t]{@{}l@{}}
                            R_{k+1,t} \times 
                            \frac{P_{b(k,t)+1, k} \,\times\, w_{b(k,t)+1}}
                                 {P_{b(k,t)+1, k+1} \,\times\, w_{b(k,t)}}
                    \end{array}$
                  \EndFor
                \EndFor
                \For{$t \gets m$ to $n-1$}
                  \For{$k \gets 1$ to $m-1$}
                    \State $L_{k,t} \gets L_{k-1,t} + L_{k,t-1} \times R_{k,t-1}$
                  \EndFor
                \EndFor
                \State $L \gets 0$
                \For{$t \gets m-1$ to $n-1$}
                  \State $L \gets L + L_{m-1,t}$
                \EndFor
                \Return $L$
        \end{algorithmic}
\end{algorithm}

When combined with the idea of dynamic programming,
the recursive formulas in the propositions and the theorem give rise to an $\cO(mn)$ algorithm
for marginalising changepoints. We spell out this algorithm in Algorithm~\ref{algo:marginal}.
\begin{theorem}\label{thm:algo:correctness}
        Algorithm~\ref{algo:marginal} computes $P(x_{1:n}\,|\, z_{1:m}, w_{1:n})$,
        where $m$ is the number of changepoints and $n$ is the number of steps in given data.
        Moreover, the algorithm runs in $\cO(mn)$ time,
        if computing $P(x_j \,|\,x_{1:j-1},z_i)$ for all $1 \leq i \leq m$ and $2 \leq j \leq n$
        takes $\cO(mn)$ time (i.e., $P(x_j \,|\,x_{1:j-1},z_i)$ can be computed in $\cO(1)$
        amortised time).
\end{theorem}
\begin{proof}
        The correctness follows from Propositions~\ref{prop:algo:finalL} and \ref{prop:algo:recR}
        and Theorem~\ref{thm:algo:recL}.
        We analyse the run time as follows.
        The line 3 computes the RHS in $\cO(n)$.
        The line 6 runs in $\cO(mn)$ by the assumption.
        In the rest of the algorithm,
        nested loops and other loops 
        iterate $\cO(mn)$ times,
        and each line inside the loops runs in $\cO(1)$.
        So, the algorithm runs in $\cO(mn)$.
\end{proof}
\begin{theorem}\label{thm:algo:differentiability}
        When $P(x_j \,|\,x_{1:j-1},z_i)$ is differentiable with respect to $x_{1:j}$ and $z_i$, the result of 
        Algorithm~\ref{algo:marginal} is also differentiable with respect to $x_{1:n}$ and $z_{1:m}$, and can be 
        computed by applying automated differentiation to the algorithm. 
\end{theorem}
\begin{proof}
        When $P(x_j \,|\,x_{1:j-1},z_i)$ is differentiable with respect to $x_{1:j}$ and $z_i$, the likelihood
        $P(x_{1:n}\,|\, z_{1:m}, w_{1:n})$ is differentiable with respect to $x_{1:n}$ and $z_{1:m}$. So, 
        the correctness of Algorithm~\ref{algo:marginal} in Theorem~\ref{thm:algo:correctness} implies
        the claimed differentiability. The other claim about the use of automated differentiation holds
        because Algorithm~\ref{algo:marginal} does not use any non-differentiable operations
        such as if statements.
\end{proof}

\subsection{Computation of normalising constant $W$}
\label{sec:norm-const}

So far we have assumed that weights $w_{1:n}$ are fixed
and the normalising constant $W$ for $P(\tau_{0:m} \,|\, w_{1:n})$ is known.
We now discharge the assumption. We present an algorithm 
for computing $W$ for given $w_{1:n}$. The algorithm uses dynamic programming,
runs in $\cO(mn)$ time, and is differentiable: the gradient of 
$W$ with respect to $w_{1:n}$ can be computed by applying automated differentiation
to the algorithm.

For all $k$ and $t$ with $1 \leq k < m$ and
$0 \leq t < n$, let
$S_{k,t} \defeq \sum_{\tau_{0:k},\, \tau_k \leq t\,} \prod_{i=1}^k w_{\tau_i}$
and $S_{0,t} \defeq 1$.
Note that $W = S_{m-1,n-1}$. So, it suffices to design an 
algorithm for computing $S_{k,t}$.  The next proposition describes how to do it.
\begin{proposition}
  \label{prop:norm-const:rec}
  For all $k,t$ with $1 \leq k < m$
  and $k \leq t < n$, we have $S_{k,t} =  S_{k,t-1} + S_{k-1,t-1} \times w_t$ and $S_{k,k-1}=0$.
\end{proposition}

The recurrence relation for $S_{k,t}$ in Proposition~\ref{prop:norm-const:rec} 
yields a dynamic-programming algorithm for computing $W$ that runs in $\cO(mn)$ time.
The standard implementation of the algorithm does not use any non-differentiable operations.
So, its gradient can be computed by automated differentiation.

With this result at hand, we remove the assumption that
weights $w_{1:n}$ are fixed
and the normalising constant $W$ for $P(\tau_{0:m} \,|\, w_{1:n})$ is known a priori.
Algorithm~\ref{algo:marginal} no longer receives $W$ as its input.
It instead uses the algorithm described above
and computes $W$ from given $w_{1:n}$ before starting line~1.
Since the computation of $W$ takes $\cO(mn)$ time and can be differentiated
by automated differentiation, all the aforementioned results on Algorithm~\ref{algo:marginal}
(Theorems~\ref{thm:algo:correctness} and~\ref{thm:algo:differentiability})
still hold, and can be extended to cover the differentiability
of Algorithm~\ref{algo:marginal} with respect to $w_{1:n}$.


\section{Learning Model Parameters}
\label{sec:learning}

Our algorithm can extend the scope of gradient-based methods for posterior inference and model learning such that they apply to changepoint models despite their non-differentiability. In this section, we explain the model-learning application. We consider state-space models with changepoints that use neural networks. The goal is to learn appropriate neural-network parameters from given data.

We consider a special case of the model described in \S\ref{sec:result} that satisfies the following conditions.
\begin{enumerate}
        \item The latent parameter $z_i$ at $i \in \{1,\ldots,m\}$ has the fixed value $e_i$ in $\{0,1\}^m$ that has $1$ at the $i$-th position and $0$ everywhere else. Formally, this means that the prior $P(z_{1:m})$ is the Dirac distribution at $(e_1,e_2,\ldots,e_m)$.
        \item The random variable $x_j$ at $j \in \{1,\ldots,n\}$ consists of two parts,
          $x_j^S \in \cX_S$ for the latent state and $x_j^O \in \cX_O$ for the observed value.
          Thus, $x_j = (x_j^S,x_j^O)$ and $\cX = \cX_S \times \cX_O$.
        \item The probability distribution $P_\phi(x_j \,|\,x_{1:j-1},z_i)$ is parameterised by $\phi \in \mathbb{R}^p$ for some $p$, and has the form
                \[
                        P_\phi(x_j \,|\, x_{1:j-1},z_i) = P_\phi(x_j^O \,|\,x_j^S,z_i) P_\phi(x_j^S \,|\,x_{1:j-1}^S,z_i).
                \]
        Typically, $P_\phi$ is defined using a neural network, and $\phi$ denotes the weights of the network.
\end{enumerate}
\noindent
When the model satisfies these conditions, we have
\begin{align} 
        \nonumber
        P_\phi(x_{1:n}\,|\,w_{1:n}) 
        & {} = 
        \sum_{z_{1:m}} P_\phi(x_{1:n},z_{1:m}\,|\,w_{1:n}) 
        \\
        \nonumber
        & {} = 
        \sum_{z_{1:m}} P(z_{1:m}) P_\phi(x_{1:n}\,|\,z_{1:m},w_{1:n}) 
        \\
        \label{eqn:marginal-z}
        & {} = 
        P_\phi(x_{1:n}\,|\,(z_{1:m}=e_{1:m}),w_{1:n}).
\end{align}

By the learning of model parameters, we mean the problem of finding $\phi$ for given observations $x_{1:n}^O$ that makes the log probability of the observations $\log P_\phi(x_{1:n}^O\,|\,w_{1:n})$ large. A popular approach \citep{KingmaICLR14} is to maximise a lower bound of this log probability, called ELBO, approximately using a version of gradient ascent:
\begin{equation}\label{eqn:ELBO}
        \ELBO_{\phi,\theta} \defeq 
        \mathbb{E}_{Q_\theta(x_{1:n}^S|x^O_{1:n})}\!\left[\log\frac{P_\phi(x_{1:n}^S,x^O_{1:n} | w_{1:n})}{Q_\theta(x_{1:n}^S | x_{1:n}^O)}\right]
\end{equation}
where $Q_\theta$ is an approximating distribution for the posterior $P_\phi(x_{1:n}^S \,|\,x_{1:n}^O,w_{1:n})$
and $\theta \in \mathbb{R}^q$ denotes the parameters of this distribution, typically the weights of a neural network used to implement $Q_\theta$.

Our marginalisation algorithm makes it possible to optimise $\ELBO_{\theta,\phi}$ in \eqref{eqn:ELBO} by an efficient stochastic gradient-ascent method based on the so called reparameterisation trick \cite{KingmaICLR14,RezendeICML14,KucukelbirJMLR2017}.
Here we use our algorithm with fixed $m, w_{1:n}$ after setting $z_{1:m}$ to $e_{1:m}$.\footnote{$W$ is computed by the extension of our algorithm in \S\ref{sec:norm-const}. In fact, using the same extension, we can even treat $w_{1:n}$ as a part of $\phi$, and learn appropriate values for $w_{1:n}$.} So, only the $x_{1:n} = (x_{1:n}^S,x_{1:n}^O)$ part of its input may vary. To emphasise this, we write 
$\algo(\phi,x_{1:n}^S,x_{1:n}^O)$ for the result of the algorithm. 
Also, we make the usual assumption of the reparameterisation trick:
there are a $\theta$-independent distribution $Q(\epsilon)$ and a differentiable function $T_\theta(\epsilon,x^O_{1:n})$ such that $T_\theta(\epsilon,x^O_{1:n})$ for $\epsilon \sim Q(\epsilon)$ is distributed according to $Q_\theta(x_{1:n}^S \,|\,x_{1:n}^O)$. The next theorem shows that the gradient of $\ELBO$ can be estimated by computing the gradient through the execution of our algorithm via automated differentiation.

\begin{theorem}
        \label{thm:algo:differentiability:learning}
If $P_\phi(x_j \,|\,x_{1:j-1},z_i)$ is differentiable with respect to $x_{1:j}$ and $\phi$,
so is $\algo(\phi,x^S_{1:n},x_{1:n}^O)$. In that case, the gradient can be computed by automated differentiation.
\end{theorem}
\begin{proof}
        When $P_\phi(x_j \,|\,x_{1:j-1},z_i)$ is differentiable with respect to $x_{1:j}$ and $\phi$, the likelihood
        $P_\phi(x_{1:n}\,|\, z_{1:m}, w_{1:n})$ is differentiable with respect to $x_{1:n}$ and $\phi$. Thus,
        the correctness of Algorithm~\ref{algo:marginal} in Theorem~\ref{thm:algo:correctness} implies
        the claimed differentiability. The other claim about the use of automated differentiation comes
        from the fact that Algorithm~\ref{algo:marginal} does not use any non-differentiable operations.
\end{proof}

\begin{theorem}
When $P_\phi(x_j \,|\,x_{1:j-1},z_i)$ is differentiable with respect to $x_{1:j}$ and $\phi$ for all $j$ and $i$, 
\[
        \widehat{\MaRep} \defeq 
        \nabla_{\phi,\theta} \log\frac{\algo(\phi,T_\theta(\epsilon,x^O_{1:n}),x^O_{1:n})}{Q_\theta(T_\theta(\epsilon,x^O_{1:n})\,|\,x_{1:n}^O)}
        \ \ \, \text{for } \epsilon \sim Q(\epsilon)
\]
is an unbiased estimate for $\nabla_{\phi,\theta}\ELBO_{\phi,\theta}$, and can be computed via automated differentiation.
\end{theorem}
\begin{proof} 
$P_\phi(x_{1:n}^S,x^O_{1:n}|w_{1:n}) = P_\phi(x_{1:n}|(z_{1:m}{=}e_{1:m}),w_{1:n})$ by \eqref{eqn:marginal-z}. The RHS of the equation equals $\algo(\phi,x^S_{1:n},x^O_{1:n})$ by Theorem~\ref{thm:algo:correctness} and the definition of $\algo$.
So, $\ELBO_{\phi,\theta} = \mathbb{E}_{Q_\theta(x_{1:n}^S|x_{1:n}^O)}\left[\log\frac{\algo(\phi,x^S_{1:n},x^O_{1:n})}{Q_\theta(x_{1:n}^S \,|\,x_{1:n}^O)}\right]$.
The usual unbiasedness argument of the reparameterisation trick and the differentiability of $\algo(\phi, x_{1:n}^S,x^O_{1:n})$ with respect to $x_{1:n}$ and $\phi$ (Theorem~\ref{thm:algo:differentiability:learning}) give the claimed conclusion.
\end{proof}

\section{Experimental Evaluation}
\label{sec:empirical}


As mentioned, another important application of our algorithm is posterior inference.
In this section, we report the findings from our experiments with this application,
which show the benefits of having an efficient 
differentiable marginalisation algorithm for posterior inference.



Hamiltonian Monte Carlo (HMC) \citep{Duane87,neal2011mcmc} is one of the most effective algorithms
for sampling from posterior distributions, especially on high dimensional spaces. However, it
cannot be applied to models with changepoints directly. This is because HMC requires that a model have a differentiable density,
but changepoint models do not meet this requirement due to discrete changepoints. 

One way of addressing this non-differentiability issue 
is to use our algorithm and marginalise changepoints. Since our algorithm is differentiable, the resulting models have differentiable densities, and we can analyse their posteriors using HMC.
We tested this approach experimentally, aiming at answering the following questions:
\begin{itemize}
\item
  RQ1 (Speed):
  How fast is our marginalisation algorithm when used for HMC? 
\item
  RQ2 (Sample quality):
  Is HMC with our marginalisation algorithm better at generating good posterior samples
  than other Markov Chain Monte Carlo (MCMC) algorithms that do not use gradients
  nor marginalisation?
\end{itemize}

\begin{figure}[t]
	\centering
	\begin{subfigure}[b]{.49\columnwidth}
      \centering%
	  \includegraphics[width=\columnwidth]{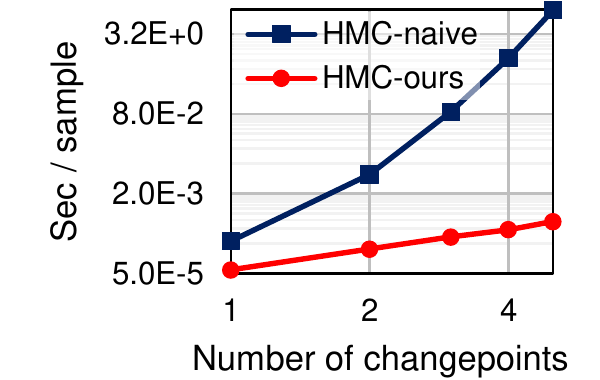}
      \captionsetup{justification=centering}
	  \caption{
        Time per sample vs\@. $m$ \\ when $n=50$.
      }
	  \label{fig:time-per-sample-by-m}
	\end{subfigure}
    \,
	\begin{subfigure}[b]{.49\columnwidth}
      \centering%
	  \includegraphics[width=\columnwidth]{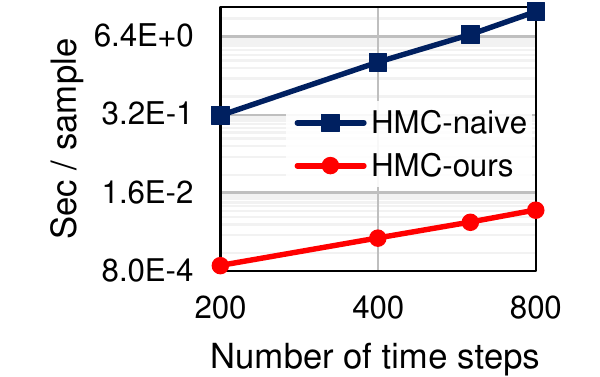}
      \captionsetup{justification=centering}
	  \caption{
        Time per sample vs\@. $n$ \\ when $m=2$.
      }
	  \label{fig:time-per-sample-by-n}
	\end{subfigure}
	\caption{
      Computation time of $\HMCnaive$ and $\HMCours$
      for different $m$ and $n$.
	  The $x$- and $y$-axes are in log-scale.
    }
	\label{fig:efficincy-comparison}
\end{figure}


We evaluated different inference algorithms on synthetic and real-world data for $\cX = \R$.
The synthetic data were generated as follows:
we fixed
parameters $(n, m^*, \mu^*_{1:m^*},$ $ \sigma^*_{1:m^*}, \tau^*_{0:m^*})$,
and then sampled each data point $x_j$ ($1 \leq j \leq n$) in the $i$-th segment 
(i.e., $\tau^*_{i-1} < j \leq \tau^*_i$) independently from a Gaussian distribution with mean $\mu^*_i$ and
standard deviation $\sigma^*_i$.
The changepoint model for analysing the synthetic data is:
$m=m^*$, $\cZ=\R\times\R_{+}$, $w_{1:n}=(1,\dots,1)$, $P(z_i\,{=}\,(\mu_i,\sigma_i)) =
\textrm{Normal}(\mu_i|5,10) \times \textrm{LogNormal}(\sigma_i|0,2)$, and
$P(x_j|x_{1:j-1},z_i\,{=}\,(\mu_i,\sigma_i)) = \textrm{Normal}(x_j|\mu_i,\sigma_i)$.
For the real-world application, we used well-log data \cite{fearnhead2006exact},
whose data points 
represent some physical quantity measured by a probe diving in a wellbore.
We took a part of the well-log data
by removing outliers and choosing 1000 consecutive data points
(Figure~\ref{fig:real-data-RQ2}).
The changepoint model for the well-log data is the same as the above except:
$m=13$ and
$P(z_i\,{=}\,(\mu_i,\sigma_i)) =
\textrm{Normal}(\mu_i|120000,20000) \times \textrm{LogNormal}(\sigma_i|8.5,0.5)$.

Our goal is to infer the posterior distribution of latent parameters $z_{1:m}$ and changepoints $\tau_{0:m}$.
For this task,
we compared four posterior-inference algorithms: $\HMCnaive$, $\HMCours$, $\IPMCMC$, and $\LMH$.
$\HMCnaive$ (resp\@. $\HMCours$) generates samples as follows:
it forms a probabilistic model $P(x_{1:n},z_{1:m}\,|\,w_{1:n})$
where $\tau_{0:m}$ are marginalised out
by a naive marginalisation scheme (resp\@. by our marginalisation algorithm);
samples $z_{1:m}$ from $P(z_{1:m}\,|\,w_{1:n},x_{1,n})$
by running HMC on $P(x_{1:n},z_{1:m}\,|$ $\,w_{1:n})$;
finally samples $\tau_{0:m}$ from $P(\tau_{0:m} \,|\, w_{1:n},x_{1:n}, z_{1:m})$
using dynamic programming.
$\IPMCMC$ and $\LMH$ jointly sample $z_{1:m}$ and $\tau_{0:m}$ from
$P(z_{1:m},\tau_{0:m}\,|\,w_{1:n},x_{1:n})$
by running the variants of the Metropolis-Hastings algorithm called interacting particle MCMC 
(IPMCMC) \cite{rainforth2016interacting} and lightweight Metropolis-Hastings (LMH) \citep{Wingate2011}, respectively. 
IPMCMC and LMH are applicable to models with discrete or non-differentiable random variables. They neither exploit gradients nor marginalise out any random variables.

For $\HMCnaive$ and $\HMCours$,
we used
the No-U-Turn Sampler (NUTS) \citep{hoffman2014no} 
in PyStan \citep{carpenter2017stan} with default hyper-parameters, except for $\text{adapt\_delta}\,{=}\,0.95$.
For $\IPMCMC$ and $\LMH$, we used the implementations in Anglican \citep{wood-aistats-2014,tolpin2016design} 
with default hyper-parameters, except for the following $\IPMCMC$ setup: $\text{number-of-nodes}\,{=}\,8$ for both
the synthetic and well-log data, and $\text{pool}\,{=}\,8$ for the well-log data.
For RQ1, we compared the time taken to generate a single posterior sample by 
$\HMCnaive$ and $\HMCours$.
%
For RQ2, we compared the quality of posterior samples from 
$\HMCours$, $\IPMCMC$, and $\LMH$,
by means of the following quantities:
estimates of the first and second moments,
the Gelman-Rubin convergence statistic ($\hat{R}$)
\citep{gelman1992inference,brooks1998general},
and effective sample size (ESS).
The experiments were performed on a Ubuntu 16.04 machine 
with Intel i7-7700 CPU with 16GB of memory.

\begin{figure}[t]
	\centering
	\begin{subfigure}[b]{\columnwidth}
      \centering
	  \includegraphics[width=\textwidth]{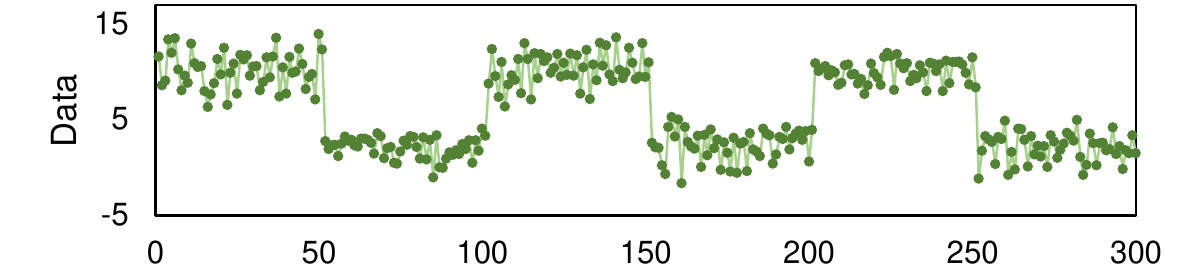} 
      \caption{Synthetic data.}
	  \label{fig:synthetic-data-RQ2}
	\end{subfigure}
    \\[0.3em]
	\begin{subfigure}[b]{\columnwidth}
      \centering
	  \includegraphics[width=\textwidth]{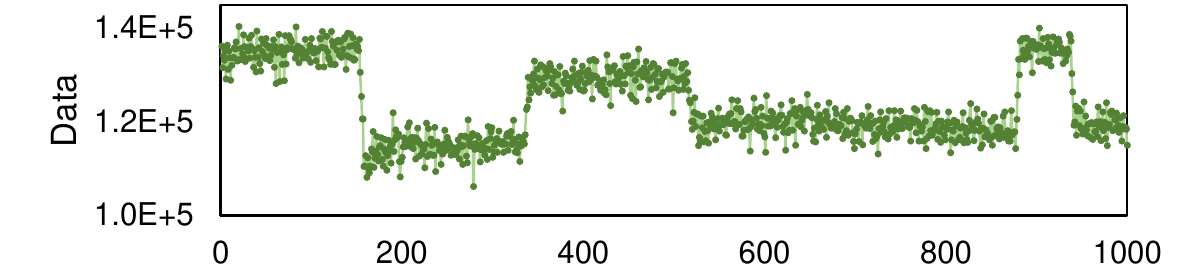}
      \caption{Part of well-log data.}
	  \label{fig:real-data-RQ2}
	\end{subfigure}
	\caption{Synthetic and real-world data for RQ2. The $x$-axis represents time steps.}
	\label{fig:data-RQ2}
\end{figure}

\vspace{1.5mm}
\noindent{\bf Results for RQ1.}\
We measured the average time per sample taken by $\HMCnaive$ and $\HMCours$ for different
numbers of changepoints and time steps:
for fixed $n=50$, we varied $m^*=m$ from $1$ to $5$,
and for fixed $m^*=m=2$, we varied $n\in\{200,400,600,800\}$. 
The details of the parameter values we used appear in the Appendix. 
We ran five independent runs of the NUTS algorithm, and averaged the time spent without burn-in samples.

Figures~\ref{fig:time-per-sample-by-m} and~\ref{fig:time-per-sample-by-n}
show how the time depends on $m$ and $n$, respectively, in the two approaches.
In the log-log plots,
$\log(\text{time})$ of $\HMCours$ is linear in both $\log m$ and $\log n$,
due to its time complexity $\cO(mn)$.
On the other hand,
$\log(\text{time})$ of $\HMCnaive$ is exponential in $\log m$,
and linear in $\log n$ yet with a slope nearly two times larger than that for $\HMCours$,
because of its time complexity $\cO(n^m)$.
Overall, the results show that $\HMCnaive$ quickly becomes infeasible as the number of changepoints or time steps
grows, but $\HMCours$ avoids such an upsurge by virtue of having the linear relationship between the two varying
factors and the time per sample.

\vspace{1.5mm}
\noindent{\bf Results for RQ2.}\
Figure~\ref{fig:data-RQ2} shows the synthetic and real-world data used in answering RQ2.
The synthetic data was generated with parameters
$n=300$, $m^*=6$, $\mu^*_{1:m^*}=(10,$ $2,10,2,10,2)$,
$\sigma^*_{1:m^*}=(1.8,1.1,1.7,1.5,1.2,1.3)$, and 
$\tau^*_{0:m^*}$ $=(0,50,$ $100,150,200,250,300)$.

For each chain of $\HMCours$, we generated 30K samples with random initialisation (when possible) after burning in 1K samples.
We computed $\hat{R}$ and ESS for each latent parameter and changepoint using three chains, and repeated this five times as the 
$\hat{R}$ and ESS results varied across different runs.
We also estimated the sum of the first moments of $(z_{1:m}, \tau_{1:m-1})$ and that of the second moments
of them using the same $15$ chains.\footnote{%
  Concretely, we estimated 
  $\mathbb{E}_{P(z_{1:m},\tau_{1:m}\,|\,x_{1:n})}
  \big[\big(\sum_{i=1}^m \mu_{i}+\sigma_{i}+\tau_i\big)-n\big]$ and
  $\mathbb{E}_{P(z_{1:m},\tau_{1:m}\,|\,x_{1:n})}
  \big[\big(\sum_{i=1}^m \mu_{i}^2+\sigma_{i}^2+\tau_i^2\big)-n^2\big]$.
}
The same setup was applied to $\IPMCMC$ and $\LMH$ except the following:
since they sample faster than $\HMCours$, we let $\IPMCMC$ and $\LMH$ generate
270K (resp. 1855K) and 200K (resp. 1750K) samples, respectively, for synthetic data (resp. well-log data) so that every
run of them spends more time than the corresponding slowest $\HMCours$ run.

\begin{figure}[t]
  \centering
  \includegraphics[width=\columnwidth]{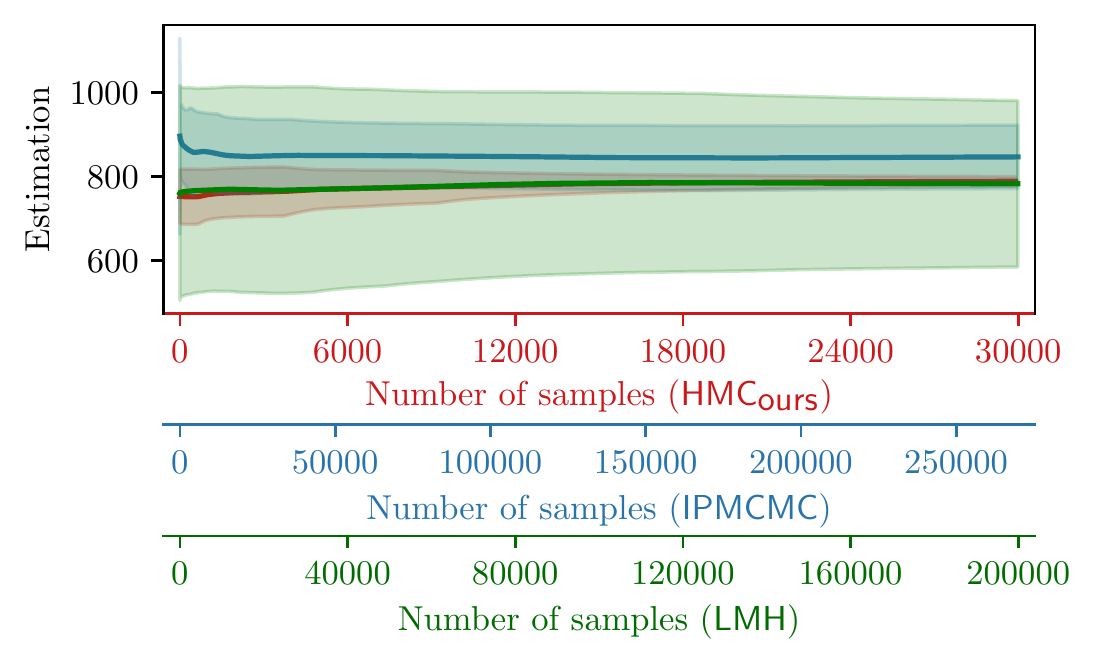}%
  \caption{%
    Convergence plots for estimating the sum of the first moments for synthetic data,
    with $\HMCours$ (red), $\IPMCMC$ (blue), and $\LMH$ (green).
    Each $x$-axis represents the number of samples
    generated by the corresponding procedure,
    and the $y$-axis denotes estimated values.
    Each darker line shows the mean value at each point,
    and the corresponding error band around it shows the standard deviation.
  } 
  \label{fig:estimation-1st-moments}
\end{figure}

We first discuss results on synthetic data.
Figure~\ref{fig:estimation-1st-moments} shows the estimates for the sum of the first moments by $\HMCours$, $\IPMCMC$,
and $\LMH$.
$\HMCours$ shows a gradual trend towards convergence, while $\IPMCMC$ and $\LMH$ exhibit substantial variation across runs
without convergence.
We obtained similar results for the sum of the second moments
(see the Appendix). 
Table~\ref{tab:moments-estimation-additional} shows the ranges of the time
taken by the Markov chains in Figure~\ref{fig:estimation-1st-moments-well-log},
and the ranges of the estimates from the chains.

Figure~\ref{fig:boxplot-rhat} shows the $\hat{R}$ values from $\HMCours$, $\IPMCMC$, and $\LMH$.
Three out of five experiments with $\HMCours$ were satisfactory in the sense that the $\hat{R}$ statistics for
\textit{all} the latent ($z_{1:m}$, $\tau_{1:m-1}$) were between $0.9$ and $1.1$.
Though the $\hat{R}$ statistics for some of the latent were over
$1.1$ in the other two experiments, most of the $\hat{R}$ values were less than or close to $1.1$.
On the other hand, none of the $\IPMCMC$ and $\LMH$ experiments placed $\hat{R}$ values for all the latent,
within the interval. Also the values were farther from the interval.

Figure~\ref{fig:boxplot-ess} shows $\ln(\text{ESS})$ from $\HMCours$, $\IPMCMC$, and $\LMH$
in a similar manner. 
$\HMCours$ produced significantly higher ESS values than $\LMH$, demonstrating that $\HMCours$ draws samples
more effectively than $\LMH$ within a fixed amount of time.
However, $\HMCours$ was not superior in ESS to $\IPMCMC$ 
despite the excellence in $\hat{R}$. We conjecture that this is due to $\IPMCMC$ running eight parallel nodes independently, each with two particles to propose samples.

\begin{table}[t]
  \caption{The ranges of the time (sec) taken by the three approaches and the ranges of the estimates computed by them, for synthetic data.
	For the estimated sum of the first/second moments (i.e., the third/fourth
	row), we computed the values at $510.8$ sec (the minimum time taken among all the runs) from each Markov chain, 
	assuming that generating each sample (in a chain) took an equal amount of time.}
\centering
\begin{adjustbox}{width=\columnwidth,center}
  \aboverulesep=0.3ex
  \belowrulesep=0.3ex
	\begin{tabular}{cccc}
      \toprule
      &$\HMCours$&$\IPMCMC$&$\LMH$
      \\ \midrule
	  Time&$[510.8,1043.9]$&$[1053.2,1170.6]$&$[1092.2,1105.2]$
      \\ \midrule
	  1st&$[746.9,794.2]$&$[651.3,993.8]$&$[415.4,1164.0]$
      \\ \midrule 
	  2nd&\makecell{$[1.28$E+$05,$\\$1.39$E+$05]$}&\makecell{$[1.12$E+$05,$\\$1.89$E+$05]$}
	  &\makecell{$[3.56$E+$04,$\\$2.62$E+$05]$}
      \\ \bottomrule
	\end{tabular}
\end{adjustbox}
\label{tab:moments-estimation-additional}
\end{table}

\begin{figure*}[t]
	\centering
	\begin{subfigure}[b]{.8\textwidth}
                \centering
		\includegraphics[width=\textwidth]{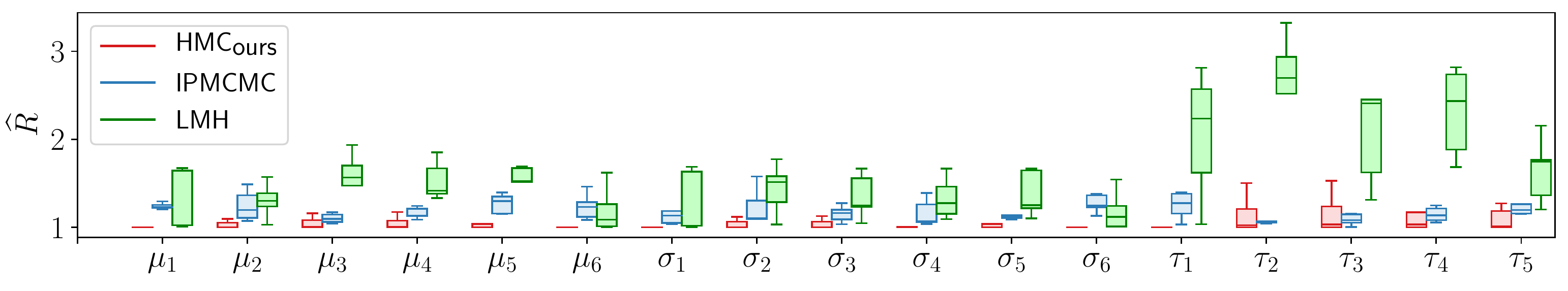}
		\caption{$\hat{R}$.}
		\label{fig:boxplot-rhat}
	\end{subfigure}
    \\
	\begin{subfigure}[b]{.8\textwidth}
                \centering
		\includegraphics[width=\textwidth]{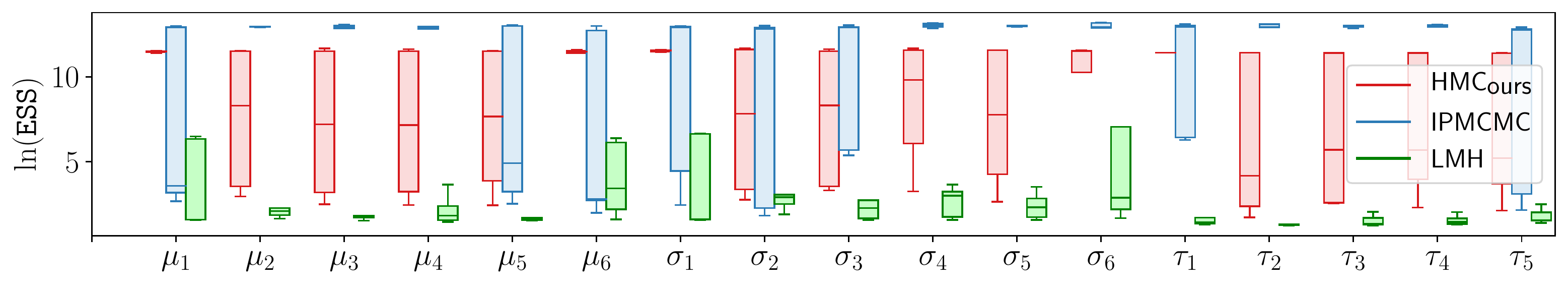}
		\caption{$\ln(\text{ESS})$.}
		\label{fig:boxplot-ess}
	\end{subfigure}
    \captionsetup{justification=centering}
	\caption{%
      $\hat{R}$ and ESS
      from $\HMCours$ (red), $\IPMCMC$ (blue), and $\LMH$ (green)
      for synthetic data.
      \\
      The $x$-axis denotes the latent parameters and changepoints,
      and the $y$-axis $\hat{R}$ or $\ln(\text{ESS})$ values.
    }
	\label{fig:rhat-ess-results}
\end{figure*}

\begin{figure}[t]
  \centering
  \includegraphics[width=\columnwidth]{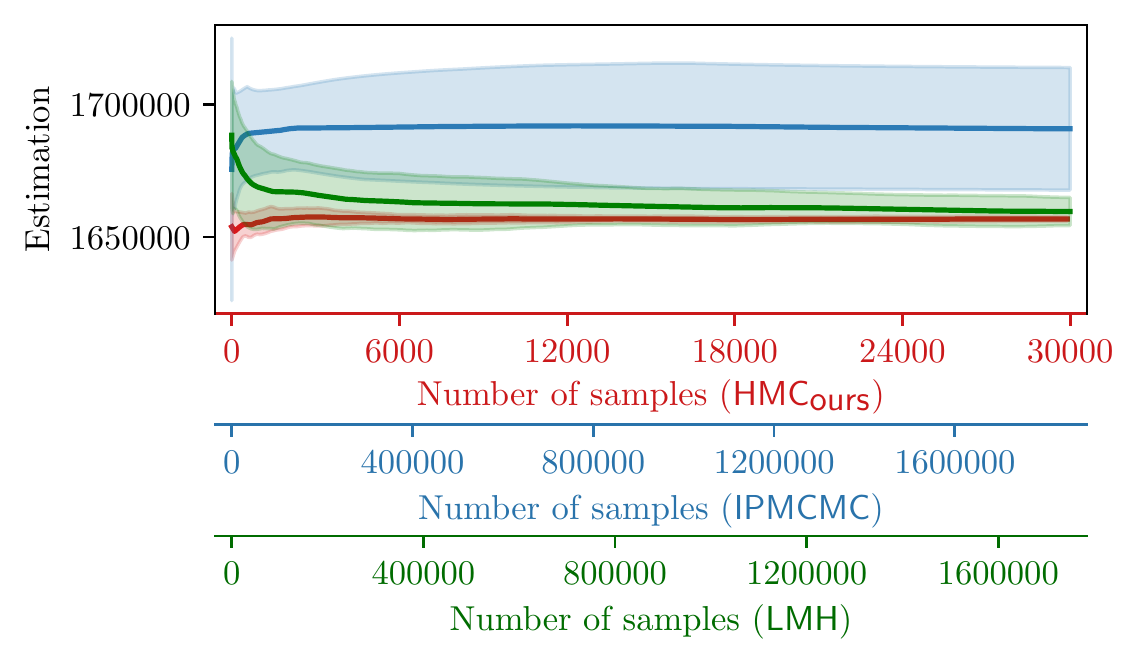}
  \caption{%
    Convergence plots for estimating the sum of the first moments
    for well-log data.
    See the caption of Figure~\ref{fig:estimation-1st-moments} for details.
  }
  \label{fig:estimation-1st-moments-well-log}
\end{figure}


  For well-log data,
  $\HMCours$ similarly outperformed the other two in terms of
  convergence, $\hat{R}$, and ESS; $\hat{R}$ for \textit{all} the latent ($z_{1:m}$, 
  $\tau_{1:m-1}$) were between $0.9$ and $1.1$ in \textit{all} five experiments. One exception is that $\IPMCMC$ showed much higher ESS than $\HMCours$, although it failed to converge
  (Figure~\ref{fig:estimation-1st-moments-well-log}). We think that this is again due to $\IPMCMC$'s tendency of generating independent samples. The full results are in the Appendix. 


  
  We remark that $\HMCours$ performed poorly on well-log data with the same model but smaller $m$
  (e.g., 6 instead of 13). According to our manual inspection, this was because $\HMCours$ in this case got stuck at
  some modes, failing to generate samples from other modes.
  We think that increasing $m$ (up to some point) lessens the barriers between modes in the marginal space; for large enough $m$, only a small amount of 
  $P(z_{1:m} \,|\, x_{1:n},w_{1:n})$ should be reduced to change some of $z_{1:m}$.
  One practical lesson is that having enough changepoints may 
  help analyse real-world data using Bayesian inference and changepoint models.

\section{Related Work and Conclusion}
\label{sec:related}

{\bf Related work.}
Modelling and reasoning about time-series data with changepoints is a well-established topic in statistics and machine learning \citep{eckley2011analysis,2018arXiv180100718T}.
We discuss two lines of research most relevant to ours. The first is the work by Fearnhead and his colleagues \cite{fearnhead2006exact,fearnhead2005exact,fearnhead2007line}, which is further extended to multi-dimensional time-series data \cite{xuan2007bayesian}. 
\citet{fearnhead2006exact} proposed an $\cO(n^2)$-time algorithm for generating changepoint positions from the posterior of a given 
changepoint model in a particular form, where $n$ is the number of time steps. Their algorithm also uses a form of dynamic 
programming on certain recursive formulas, but it does not target at marginalisation. Its conversion for marginalisation is possible, but
inherits this $\cO(n^2)$ time complexity. The other work is \citet{chib1995marginal}'s technique for estimating the model evidence of changepoint models \cite{chib1998marginal}, whose
properties, such as sensitivity on chosen parameters, is analysed by \citet{bauwens2012marginal}. The technique is based on Gibbs sampling, and it is unclear whether the technique leads to a differentiable algorithmic component that can be used in the context of gradient-based algorithms.

The observation that the summation version of dynamic programming is differentiable is a folklore. For instance, \citet{Eisner16} points out the differentiability of the inside algorithm, which is a classic dynamic-programming-based algorithm in natural language processing (NLP). He then explains how to derive several well-known NLP algorithms by differentiating the inside algorithm or its variants. However, we do not know of existing work that uses such dynamic programming algorithms for the type of application we consider in the paper: converting non-differentiable models to differentiable models via marginalisation in the context of posterior inference and model learning. The optimisation version of dynamic programming is not differentiable, and its differentiable relaxation has been studied recently \cite{Corr-ICLR19,MenschB18}.

\vspace{1.5mm}
\noindent{\bf Conclusion.}\ 
We presented a differentiable $\cO(mn)$-time algorithm for marginalising changepoints in time-series models, where $m$ is the number of changepoints and $n$ the number of time steps. The algorithm can be used to convert a class of non-differentiable time-series models to differentiable ones, so that the resulting models can be analysed by gradient-based techniques. We described two applications of this conversion, posterior inference with HMC and model-parameter learning with reparameterisation gradient estimator, and experimentally showed the benefits of using the algorithm in the former posterior-inference application.

\vspace{1.5mm}
\noindent{\bf Acknowledgements.}
The authors were supported by the Engineering Research Center Program through the National Research Foundation of Korea (NRF) funded by the Korean Government MSIT (NRF-2018R1A5A1059921), and also by Next-Generation Information Computing Development Program through the National Research Foundation of Korea (NRF) funded by the Ministry of Science, ICT (2017M3C4A7068177).

\bibliography{refs}
\bibliographystyle{aaai}

\appendix

\section{Appendix}

\subsection{Proof of Proposition~6}
\noindent{\bf Proposition 6.}\ 
{\it For all $k,t$ with $1 \leq k < m$
  and $k \leq t < n$, we have $S_{k,t} =  S_{k,t-1} + S_{k-1,t-1} \times w_t$ and $S_{k,k-1}=0$.}

\begin{proof}
We focus on the first equality. 
\begin{align*}
  S_{k,t}
  &= \sum_{\tau_{0:k},\, \tau_k <t\,} \prod_{i=1}^k w_{\tau_i}
  + \sum_{\tau_{0:k},\, \tau_k =t\,} \prod_{i=1}^k w_{\tau_i}
  \\
  &= \sum_{\tau_{0:k},\, \tau_k < t\,} \prod_{i=1}^k w_{\tau_i}
  + w_t \times \sum_{\tau_{0:k-1},\, \tau_{k-1} < t\,} \prod_{i=1}^{k-1} w_{\tau_i}
  \\
  &= S_{k, t-1} + w_t \times S_{k-1, t-1}\,.
\end{align*}
\end{proof}

\subsection{Setup for data generation for RQ1 (\S4)}

\begin{table*}[t]
\caption{Parameter values used in data generation for RQ1.}
\centering
       \begin{tabular}{|c|r|r|r|r|r|}
           \hline
           &\multicolumn{1}{c|}{$n$}&\multicolumn{1}{c|}{$m^*$}&\multicolumn{1}{c|}{$\mu^*_{1:m^*}$}&
           \multicolumn{1}{c|}{$\sigma^*_{1:m^*}$}&\multicolumn{1}{c|}{$\tau^*_{0:m^*}$} \\ \hline
           \multirow{5}{*}{Varying $m^*$}&$50$&$1$&$(9.0,2.0)$&$(1.5,1.5)$&$(0,8,50)$ \\
           &$50$&$2$&$(9.0,2.0,9.0)$&$(1.5,1.5,1.5)$&$(0,8,16,50)$ \\ 
           &$50$&$3$&$(9.0,2.0,9.0,2.0)$&$(1.5,1.5,1.5,1.5)$&$(0,8,16,24,50)$ \\
           &$50$&$4$&$(9.0,2.0,9.0,2.0,9.0)$&$(1.5,1.5,1.5,1.5,1.5)$&$(0,8,16,24,32,50)$ \\
           &$50$&$5$&$(9.0,2.0,9.0,2.0,9.0,2.0)$&$(1.5,1.5,1.5,1.5,1.5,1.5)$&$(0,8,16,24,32,40,50)$ \\ \hline
           \multirow{5}{*}{Varying $n$}&$200$&$2$&$(8.8,2.0,7.3)$&$(1.8,1.1,1.7)$&$(0,27,80,200)$ \\
           &$400$&$2$&$(8.8,2.0,7.3)$&$(1.8,1.1,1.7)$&$(0,27,80,400)$ \\
           &$600$&$2$&$(8.8,2.0,7.3)$&$(1.8,1.1,1.7)$&$(0,27,80,600)$ \\
           &$800$&$2$&$(8.8,2.0,7.3)$&$(1.8,1.1,1.7)$&$(0,27,80,800)$ \\ \hline
       \end{tabular}
   \label{tab:detailed-setup-RQ1}
\end{table*}

See Table~\ref{tab:detailed-setup-RQ1}.

\subsection{Full results for RQ2 (\S4)}
\label{sec:appendix-2nd-moments}

See Figures~\ref{fig:estimation-2nd-moments} and \ref{fig:rhat-ess-results-well-log}.

\begin{figure}[h]
  \centering
  \begin{subfigure}[b]{\columnwidth}
    \centering
	\includegraphics[width=\columnwidth]{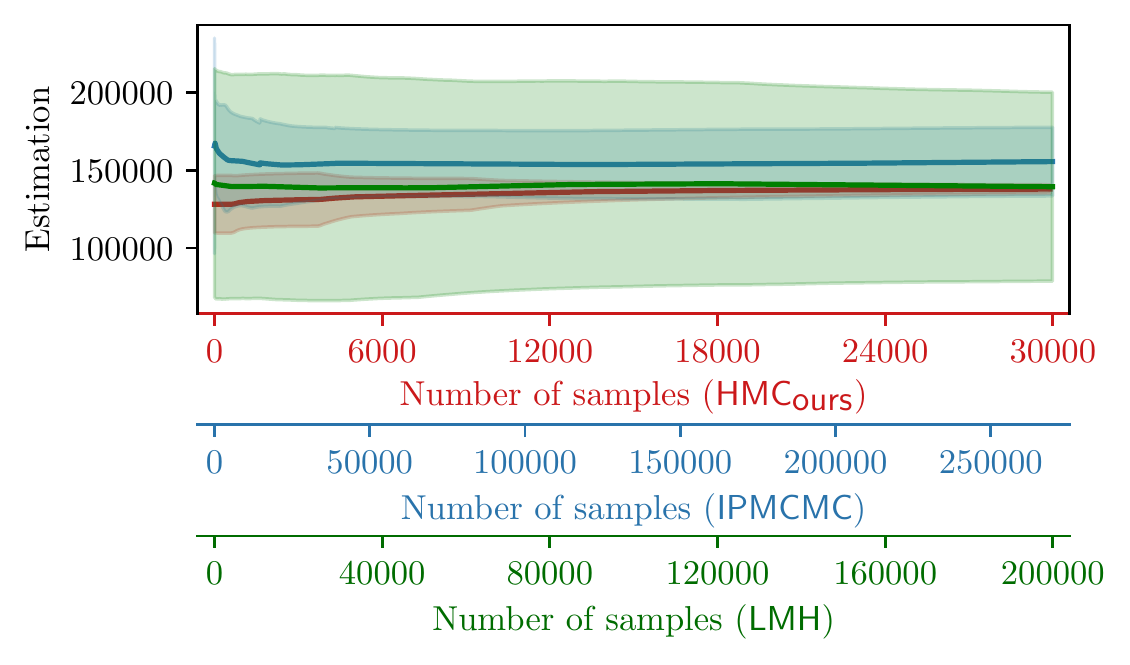}
    \caption{Synthetic data.}
  \end{subfigure}
  \\
  \begin{subfigure}[b]{\columnwidth}
    \centering
    \includegraphics[width=\textwidth]{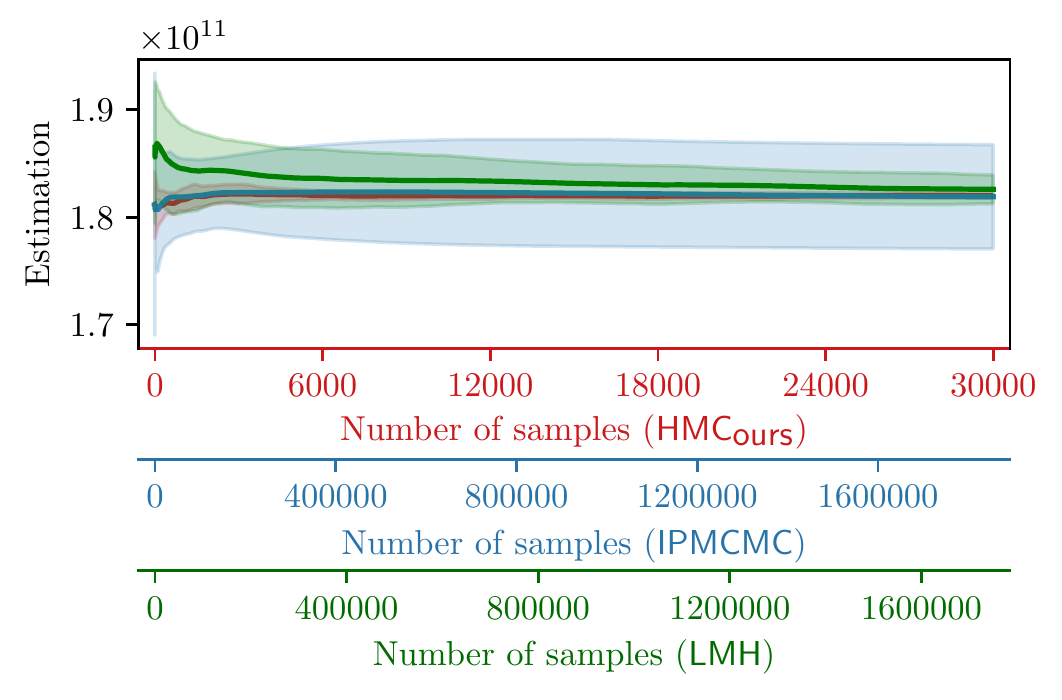}
    \caption{Well-log data.}
  \end{subfigure}
  \caption{%
    Convergence plots for estimating the sum of the second moments.
    See the caption of Figure~\ref{fig:estimation-1st-moments} for details.
  }
  \label{fig:estimation-2nd-moments}
\end{figure}

\begin{figure*}
	\centering
	\begin{subfigure}[b]{\textwidth}
                \centering
		\includegraphics[width=\textwidth]{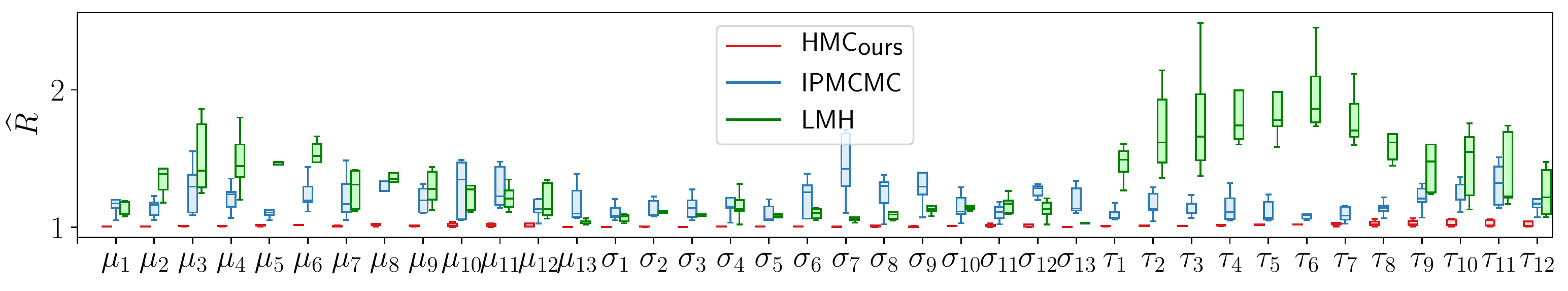}
		\caption{$\hat{R}$.}
	\end{subfigure}
    \\
	\begin{subfigure}[b]{\textwidth}
                \centering
		\includegraphics[width=\textwidth]{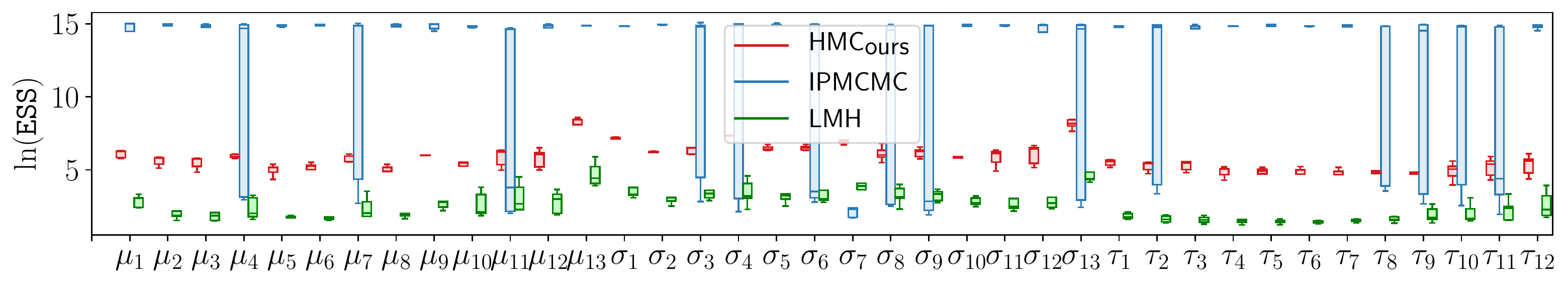}
		\caption{$\ln(\text{ESS})$.}
	\end{subfigure}
    \captionsetup{justification=centering}
	\caption{%
      $\hat{R}$ and ESS
      from $\HMCours$ (red), $\IPMCMC$ (blue), and $\LMH$ (green)
      for well-log data.
      \\
      The $x$-axis denotes the latent parameters and changepoints,
      and the $y$-axis $\hat{R}$ or $\ln(\text{ESS})$ values.
    }
	\label{fig:rhat-ess-results-well-log}
\end{figure*}

\end{document}